\documentclass[11pt,reqno]{article}
\usepackage{jmlr2e}
\usepackage{fullpage,times,graphicx,amssymb,amsmath,amsfonts,bbm,psfrag,xcolor}

\usepackage{multicol}
\usepackage{graphicx}
\usepackage{amssymb}
\usepackage{amsmath}
\usepackage{amsfonts}
\usepackage{bbm}
\usepackage{booktabs}
\usepackage[table]{xcolor}  
\usepackage{psfrag}
\usepackage{xcolor}
\usepackage{fullpage}
\usepackage{epstopdf}
\usepackage{graphicx}
\usepackage{times}
\usepackage{jmlr2e}
\usepackage{commath}
\usepackage{thmtools}
\usepackage{thm-restate}
\usepackage{times}
\usepackage{graphicx}
\usepackage{natbib}
\usepackage{dirtytalk}
\usepackage{empheq}
\definecolor{ddarkbrown}{rgb}{0.5,0.2,0.05} \definecolor{bbluegray}{rgb}{0.05,0,0.5}

\usepackage[colorlinks,citecolor=bbluegray,linkcolor=ddarkbrown,urlcolor=blue,breaklinks]{hyperref}

\usepackage{subfig,float} 

\usepackage{pifont} 

\newcommand{\cmark}{\ding{51}} 
\newcommand{\xmark}{\ding{55}} 

\usepackage{mathtools}

\usepackage{algorithm,algcompatible,algpseudocode}
\usepackage{multicol,lipsum}
\algnewcommand{\Inputs}[1]{%
	\State \textbf{Inputs: \:}{#1}
}

\algnewcommand{\Output}[1]{%
	\State \textbf{Output: \:}{#1}
}
\algnewcommand{\Initialize}[1]{%
	\State \textbf{Initialize: \:}{#1}
}

\algnewcommand{\IIf}[1]{\State\algorithmicif\ #1\ \algorithmicthen}
\algnewcommand{\EndIIf}{\unskip\ \algorithmicend\ \algorithmicif}

\let \oldsection \section
\renewcommand{\section}{\vspace{3ex plus 1ex}\oldsection}

\newcommand{\BEAS}{\begin{eqnarray*}}
	\newcommand{\EEAS}{\end{eqnarray*}}
\newcommand{\BEA}{\begin{eqnarray}}
\newcommand{\EEA}{\end{eqnarray}}

\newcommand{\BEQ}{\begin{equation}}
\newcommand{\EEQ}{\end{equation}}
\newcommand{\BIT}{\begin{itemize}}
	\newcommand{\EIT}{\end{itemize}}
\newcommand{\BNUM}{\begin{enumerate}}
	\newcommand{\ENUM}{\end{enumerate}}

\newcommand{\BA}{\begin{array}}
	\newcommand{\EA}{\end{array}}

\newtheorem{mythm}{Theorem}[section]
\newtheorem{mydef}{Definition}[section]
\newtheorem{myprop}{Proposition}[section]
\newtheorem{mylem}{Lemma}[section]

\newtheorem{myremark}{Remark}[section]
\newtheorem{myassump}{Assumption}[section]

\title{A New Convergence Analysis of Plug-and-Play Proximal Gradient Descent Under Prior Mismatch}

\begin{document}	

	\author{\name Guixian Xu  \email gxx422@student.bham.ac.uk\\
		\addr School of Mathematics,\\ University of Birmingham \\ \\
         \name Jinglai Li  \email j.li.10@bham.ac.uk\\
		\addr School of Mathematics,\\ University of Birmingham \\ \\
    \name Junqi Tang  \email j.tang.2@bham.ac.uk\\
		\addr School of Mathematics,\\ University of Birmingham \\
        \\
        }

	\editor{}

	\maketitle


\begin{abstract}

In this work, we provide a new convergence theory for plug-and-play proximal gradient descent (PnP-PGD) under prior mismatch where the denoiser is trained on a different data distribution to the inference task at hand. To the best of our knowledge, this is the first convergence proof of PnP-PGD under prior mismatch. Compared with the existing theoretical results for PnP algorithms, our new results removed the need for several restrictive and unverifiable assumptions. Moreover, we derive the convergence theory for equivariant PnP (EPnP) under the prior mismatch setting, proving that EPnP reduces error variance and explicitly tightens the convergence bound. 
 
\end{abstract}

\section{Introduction}

Inverse problems involve the recovery of an unknown signal $\boldsymbol{x} \in \mathbb{R}^n$ from a set of noisy measurements $\boldsymbol{y} = \boldsymbol{A}\boldsymbol{x} + \boldsymbol{e}$, where $\boldsymbol{A} \in \mathbb{R}^{m\times n}$ is the measurement model and $\boldsymbol{e}$ is the noise. Inverse problems are often formulated and solved as optimization problems of form 
\begin{equation}\label{eq: obj}
\boldsymbol{x}^* \in \underset{\boldsymbol{x}}{\arg \min } \lambda f(\boldsymbol{x})+g(\boldsymbol{x})
\end{equation}
where $f$ is a data-fidelity term, $g$ a regularization term and $\lambda >0$ a parameter that controls the strength of the regularization. Proximal algorithms are widely used for solving Eq.$~\eqref{eq: obj}$ when the regularizer is nonsmooth. For example, the PGD is a standard approach for solving Eq.$~\eqref{eq: obj}$  via the iterations
\begin{equation}\label{eq: PGD}
\tag{PGD}
\begin{cases}
    \boldsymbol{z}_{k+1} &= \boldsymbol{x}_{k} - \lambda \nabla f(\boldsymbol{x}_{k}) \\
    \boldsymbol{x}_{k+1} &= \mathsf{prox}_{\tau g}(\boldsymbol{z}_{k+1}),
\end{cases}
\end{equation}
The second step of~\eqref{eq: PGD} relies on the proximal operator defined as
\begin{equation}\label{eq: prox}
    \mathsf{prox}_{\tau g}(\boldsymbol{z}) := \operatorname*{arg\,min}_{\boldsymbol{x}\in\mathbb{R}^n} \left\{ \frac{1}{2}\|\boldsymbol{x} - \boldsymbol{z}\|_2^2 + \tau g(\boldsymbol{x}) \right\},
\end{equation}
where $\tau > 0$ controls the influence of $g$ and we fix $\tau = 1$ in the following. 


\section{Theory}

\begin{table}[t]
\centering
\begin{tabular}{
    p{0.22\linewidth}
    >{\centering\arraybackslash}p{0.16\linewidth}
    >{\centering\arraybackslash}p{0.16\linewidth}
    >{\centering\arraybackslash}p{0.16\linewidth}
}
\toprule
\textbf{Work} &
\textbf{Method} &
\textbf{Target Denoiser?} &
\textbf{Assumptions} \\
\midrule
\multicolumn{4}{l}{\emph{Convex functions}} \\
\midrule
 \cite{schmidt2011convergence}  & PGD &  \textbf{---}  & $\epsilon$-optimal  \\
 \cite{shoushtari2022deep} & RED &  \xmark  & BD, BI \\
\midrule
\multicolumn{4}{l}{\emph{General Non-convex functions}} \\
\midrule
 \cite{xu2020provable} & PnP-PGD &  \cmark  & TD  \\
 \cite{hurault2022proximal} & PnP-PGD &  \cmark & TD  \\
 \cite{shoushtari2024prior} & PnP-ADMM &  \xmark  & BD, BI  \\
\rowcolor{green!25}
 Theorem~\ref{them: 1} & PnP-PGD &  \xmark  & BD  \\
\bottomrule
\end{tabular}
\caption{Comparison of the convergence results in the PGD literature. Here TD = using Target Denoiser, BD = Bounded Denoiser, BI = Bounded Iterates.}
\label{tab:compare-strategy}
\end{table}

This section presents a convergence analysis of PnP-PGD that accounts for the use of mismatched denoisers. It should be noted that the theoretical analysis of PGD and PnP-PGD has been previously discussed in~\citep{schmidt2011convergence, xu2020provable, hurault2022proximal,tan2024provably}. In addition, ~\cite{shoushtari2024prior} and~\cite{shoushtari2022deep} further theoretically learned about the prior mismatch in PnP-ADMM and RED with relatively restrictive and unverifiable assumptions of $\textit{bounded-iterates}$ (see $\textit{Assumption 6}$ in~\citep{shoushtari2024prior} and $\textit{Assumption 4}$ in~\citep{shoushtari2022deep}). The novelty of our work can be summarized as follows:
\begin{itemize}
    \item We analysis convergence and measure error bounds of the mismatched priors in PnP-PGD with relaxed assumptions;
    \item Our theory accommodates nonconvex data fidelities, nonconvex regularizers and expansive denoisers.
    \item We derive the convergence theory for equivariant PnP-PGD (EPnP) under the prior mismatch setting. We prove that EPnP reduces error variance and explicitly tightens the convergence bound for PnP-PGD. 
\end{itemize}

\subsection{PnP-PGD with Mismatched Denoiser}
\eqref{eq: PGD} alternates between a proximal operation on $g$ and a gradient descent step on $f$, when $f$ is differentiable. PnP methodology proposes to replace the proximal operator by a more general denosier $\mathsf{D}(\cdot)$, such as BM3D or CNN. Then, \eqref{eq: PGD} can be summarized as
\begin{equation}\label{eq: PnP-PGD}
\tag{PnP-PGD}
\begin{cases}
    \boldsymbol{z}_{k+1} &= \boldsymbol{x}_{k} - \lambda \nabla f(\boldsymbol{x}_{k}) \\
    \boldsymbol{x}_{k+1} &= \mathsf{D}_{\sigma}(\boldsymbol{z}_{k+1}),
\end{cases}
\end{equation}
where by analogy to $\tau$ in Eq.$~\eqref{eq: prox}$, we introduce the parameter $\sigma >0$ for controlling the relative strength of the denoiser $\mathsf{D}_\sigma$. \\
We denote the target distribution as $p_{\boldsymbol{x}}$, suppose that the denoiser $\mathsf{D}_\sigma$ is a MMSE estimator for the AWGN denoising problem
\begin{equation}\label{eq: AWGN}
    \boldsymbol{v} = \boldsymbol{x} + \boldsymbol{e} \quad \text{with} \quad \boldsymbol{x} \sim p_{\boldsymbol{x}}, \quad \boldsymbol{e} \sim \mathcal{N}(0, \sigma^2 \boldsymbol{I}).
\end{equation}
The MMSE denoiser is the conditional mean estimator for Eq.$~\eqref{eq: AWGN}$ and can be expressed as
\begin{equation}\label{eq: MMSE}
    \mathsf{D}_\sigma(\boldsymbol{v}) := \mathbb{E}[\boldsymbol{x}|\boldsymbol{v}] = \int_{\mathbb{R}^n} \boldsymbol{x} p_{\boldsymbol{x}|\boldsymbol{v}}(\boldsymbol{x}|\boldsymbol{v}) \, \mathrm{d}\boldsymbol{x},
\end{equation}
where $p_{\boldsymbol{x|v}}(x|v) \propto G_{\sigma}(\boldsymbol{v} - \boldsymbol{x})p_{\boldsymbol{x}}(\boldsymbol{x})$, with $G_\sigma$ denoting the Gaussian density. Since the above integral is generally intractable, in practice, the denoiser corresponds to a deep model trained to minimize the mean squared error (MSE) loss
\begin{equation}\label{eq: appro-MMSE}
\mathcal{L}(\mathsf{D}_\sigma) = \mathbb{E}\left[\|\boldsymbol{x} - \mathsf{D}_\sigma(\boldsymbol{v})\|_2^2\right].
\end{equation}
where the solution of Eq.$~\eqref{eq: appro-MMSE}$ satisfies Eq.$~\eqref{eq: MMSE}$.\\
When using a mismatched prior with a mismatched distribution $\widehat{p}(\boldsymbol{x})$, we replace the target denoiser $\mathsf{D}_\sigma$ in~\eqref{eq: PnP-PGD} with
\begin{equation*}
    \boldsymbol{x}_{k+1} = \widehat{\mathsf{D}}_\sigma (\boldsymbol{z}_{k+1}),
\end{equation*}
where $\widehat{\mathsf{D}}_\sigma$ is the mismatched MMSE denoiser.

\subsection{Theoretical Analysis}
Recently introduced Gradient-Step (GS) Denoiser~\citep{hurault2022gradient, cohen2021has} writes
\begin{equation}
	\mathsf{D}_\sigma = \nabla h_\sigma,
\end{equation}
with a potential
\begin{equation}
	h_\sigma : \boldsymbol{x} \to \frac{1}{2} \|\boldsymbol{x}\|^2 - g_\sigma (\boldsymbol{x}).
\end{equation}
We first have that if $h_\sigma$ is convex, the GS denoiser $\mathsf{D}_\sigma = \nabla h_\sigma$ is linked to the proximal operator of some function~\citep{gribonval2020characterization}
\begin{equation}
    \phi_\sigma : \mathbb{R}^n \to \mathbb{R} \cup \{ +\infty \}: \forall \boldsymbol{x} \in \mathbb{R}^n, D_\sigma(\boldsymbol{x}) \in \mathsf{prox}_{\phi_\sigma}(\boldsymbol{x}) = \mathop{\arg\min}_{\boldsymbol{z}} \frac{1}{2} \| \boldsymbol{x}-\boldsymbol{z}\|^2 + \phi_\sigma(\boldsymbol{z}).
\end{equation}
The next proposition shows that, if the residual $\text{I}d - \mathsf{D}_\sigma$ is contractive, there exists a closed-form and smooth $\phi_\sigma$ such that $\mathsf{D}_\sigma = \mathsf{prox}_{\phi_\sigma}$ is single-valued.
\begin{myprop}[\citep{hurault2022proximal}]\label{prop: 1}
    Let $\mathcal{X}$ be an open convex subset of $\mathbb{R}^n$ and $g_\sigma: \mathcal{X} \to \mathbb{R}$ a $\mathcal{C}^{k+1}$ function with $k \geq 1$ and $\nabla g_\sigma$ $L$-Lipschitz with $L < 1$. Then, for $h_\sigma : \boldsymbol{x} \to \frac{1}{2} \|\boldsymbol{x}\|^2 - g_\sigma(\boldsymbol{x})$ and $\mathsf{D}_\sigma:= \nabla h_\sigma = \text{I}d - \nabla g_\sigma$,
    \begin{itemize}
        \item[(i)] $h_\sigma$ is $(1-L)$-strongly convex and $\forall \boldsymbol{x} \in \mathcal{X}$, $J_{\mathsf{D}_\sigma}(\boldsymbol{x})$ is positive define;
        \item[(ii)] $\mathsf{D}_\sigma$ is injective, $\mathrm{Im}(\mathsf{D}_\sigma)$ is open and, $\forall \boldsymbol{x} \in \mathcal{X}$, $\mathsf{D}_\sigma(\boldsymbol{x}) = \mathsf{prox}_{\phi_\sigma}(\boldsymbol{x})$, with $\phi_\sigma: \mathcal{X} \to \mathbb{R} \cup \{+\infty \}$ defined by  
        \begin{equation}
            \phi_\sigma(\boldsymbol{x}):=\left\{\begin{array}{lc}
            \left(g_\sigma\left(\mathsf{D}_\sigma^{-1}(\boldsymbol{x})\right)\right)-\frac{1}{2}\left\|\mathsf{D}_\sigma^{-1}(\boldsymbol{x})-\boldsymbol{x}\right\|^2 
            & \text { if } \boldsymbol{x} \in \operatorname{Im}\left(\mathsf{D}_\sigma\right), \\
            +\infty & \text { otherwise },
            \end{array}\right.
        \end{equation}
        \item[(iii)] $\phi_\sigma$ is $\mathcal{C}^k$ on $\mathrm{Im}(\mathsf{D}_\sigma)$ and $\forall \boldsymbol{x} \in \mathsf{D}_\sigma$, $\nabla \phi_\sigma(\boldsymbol{x}) = \mathsf{D}_\sigma^{-1}(\boldsymbol{x}) - \boldsymbol{x} = \nabla g_\sigma (\mathsf{D}_\sigma^{-1}(\boldsymbol{x}))$;
        \item[(v)] $\nabla \phi_\sigma$ is $\frac{L}{1-L}$-Lipschitz on $\mathrm{Im}(\mathsf{D}_\sigma)$. 
    \end{itemize}
\end{myprop}

\begin{myremark}
    Note that $\phi_\sigma$ being possibly nonconvex, and also note that $\mathsf{D}_\sigma$ is possibly not nonexpansive.
\end{myremark}
In the following, we study the convergence of the PnP-PGD with a target denoiser $\mathsf{D}_\sigma = \mathsf{prox}_{\phi_\sigma}$ that corresponds to the proximal operator of a nonconnvex regularization function $\phi_\sigma$, as well as a \textbf{mismatched} denoiser $\widehat{\mathsf{D}}_\sigma$ corresponds to a deep model trained to minimize the MSE loss. For that propose, we re-target the objective function
\begin{equation}
    F_{\lambda, \sigma} := \lambda f + \phi_\sigma,
\end{equation}
where $f$ is a data-fidelity term, $\lambda$ is a regularization parameter and $\phi_\sigma$ is defined as in Proposition~\ref{prop: 1} from the function $g_\sigma$ satisfying $\mathsf{D}_\sigma = \text{I}d - \nabla g_\sigma$. The analysis relies on the following set of assumptions that serve as sufficient conditions.

\begin{myassump}\label{assump: 1}
    $f : \mathbb{R}^n \to \mathbb{R} \cup \{ + \infty \}$ differentiable with $L_f$-Lipschitz gradient. 
\end{myassump}
This assumption is a standard assumption used in nonconvex optimization.

\begin{myassump}\label{assump: 2}
    $g_\sigma: \mathcal{X} \to \mathbb{R}$ a $\mathcal{C}^{2}$ function and $\nabla g_\sigma$ $L$-Lipschitz with $L < 1$. 
\end{myassump}
Instead of requiring $\mathsf{D}_\sigma$ is contractive, Assumption~\ref{assump: 2} assumes that $Id - \mathsf{D}_\sigma$ is contractive, which is a relaxed assumption as in~\citep{ryu2019plug}. Besides, If $\nabla g_\sigma$ has Lipschitz constant $L>1$, we can introduce $0<\alpha<\frac{1}{L}$ and relax the denoising operation by replacing $D_\sigma$ with $D_\sigma^\alpha = \alpha D_\sigma + (1-\alpha)Id = Id - \alpha \nabla g_\sigma$, then we can define $\phi_\sigma^\alpha$ from $\alpha g_\sigma$ and using the modified PnP-PGD $x_{k+1} = D_\sigma^\alpha (x_k - \lambda \nabla f(x_k))$, then we can get the same convergence result of scenario with $\nabla g_\sigma$ has Lipschitz constant $L<1$.

\begin{myassump}\label{assump: 3}
    $f$ and $g_\sigma$ bounded from below, then from Proposition~\ref{prop: 1}, we get that $\phi_\sigma$ and $F_{\lambda, \sigma}$ are also bounded from below.
\end{myassump}
Assumption~\ref{assump: 3} implies that there exists $F_{\lambda, \sigma}^* > -\infty$ such that $F_{\lambda, \sigma}(\boldsymbol{x}) \geq F_{\lambda, \sigma}^*$.

\begin{myassump}\label{assump: 4}
    The denoiser $\mathsf{D}_\sigma$ and $\widehat{\mathsf{D}}_\sigma$ have the same range $\mathsf{Im}(\mathsf{D}_\sigma)$. 
\end{myassump}
This assumption that the two image denoisers have the same range is also a relatively mild assumption. Ideally, both denoisers would have the same range corresponding to the set of desired images.

\begin{myassump}\label{assump: 5}
    The $k+1$-th output of mismatched denoiser $\boldsymbol{x}_{k+1} = \widehat{\mathsf{D}}_\sigma (\boldsymbol{z}_{k+1})$ is a $\epsilon_{k+1}$-optimal solution to the proximal problem in the sence of target denoiser $\mathsf{D}_\sigma$:
    \begin{equation}
	\phi_\sigma (\boldsymbol{x}_{k+1}) + \frac{1}{2} \|\boldsymbol{x}_{k+1} - \boldsymbol{z}_{k+1} \|^2 \leq \min_{\boldsymbol{u}} \left( \phi_\sigma (\boldsymbol{u}) + \frac{1}{2} \|\boldsymbol{u} - \boldsymbol{z}_{k+1}\|^2 \right) + \epsilon_{k+1}, \quad k=0,1,\ldots
\end{equation}
where $\boldsymbol{z}_{k+1} = \boldsymbol{x}_k- \lambda \nabla f(\boldsymbol{x}_k)$.
\end{myassump}
\begin{myremark}
By Proposition~\ref{prop: 1} we get that $H(\boldsymbol{x}) = \frac{1}{2}\|\boldsymbol{x}-\boldsymbol{z}\|^2 + \phi_\sigma(\boldsymbol{x})$ is $\frac{1}{L+1}$-strongly convex (see Appendix for proof), which means
\begin{equation}
	H(\widehat{\mathsf{D}}_\sigma (\boldsymbol{z}_{k+1})) - H(\mathsf{D}_\sigma (\boldsymbol{z}_{k+1})) \geq \frac{1}{2(L+1)} \|\widehat{\mathsf{D}}_\sigma (\boldsymbol{z}_{k+1}) - \mathsf{D}_\sigma (\boldsymbol{z}_{k+1}) \|^2 
\end{equation}
where $\mathsf{D}_\sigma(\boldsymbol{z}_{k+1}) = \mathop{\arg\min}_{\boldsymbol{x}} H(\boldsymbol{x})$, then with Assumption~\ref{assump: 5} we can then get the equivalent form:
\begin{equation}
	\| \widehat{\mathsf{D}}_\sigma (\boldsymbol{z}_{k+1}) - \mathsf{D}_\sigma (\boldsymbol{z}_{k+1}) \| \leq \sqrt{2(L+1)H(\boldsymbol{x}_{k+1}) - H(\boldsymbol{x}^*)} \leq \sqrt{2(L+1)\epsilon_{k+1}},
\end{equation}
which is a common assumption as in~\cite{shoushtari2024prior, shoushtari2022deep}.
\end{myremark}
We are now ready to present our convergence result under mismatched MMSE denoisers.

\begin{mythm}\label{them: 1}
    Run PnP-PGD using a $\textbf{mismatched}$ MMSE denoiser for $t \geq 1$ iterations under Assumption~\ref{assump: 1}-\ref{assump: 5}, with $\lambda L_f < 1$, then we have
    \begin{equation}
	\min_{1 \leq k \leq t} \| \nabla F_{\lambda, \sigma} (\boldsymbol{x}_k) \|^2 \leq \frac{1}{t} \sum_{k=1}^t \| \nabla F_{\lambda, \sigma} (\boldsymbol{x}_k) \|^2 \leq \frac{1}{t} \cdot \frac{16}{1 - L_f} \left( F_{\lambda, \sigma}(\boldsymbol{x}_0) - F_{\lambda, \sigma}^* \right) + \frac{1}{t} \cdot \left( \frac{16}{1- L_f} + \frac{4}{1-L} \right) \sum_{k=1}^t \epsilon_k,
\end{equation}
where $\epsilon_{k}$ is the error term. In addition, if the sequence $\{ \epsilon_k \}_{k \geq 1}$ is summable, we have that $\|\nabla F_{\lambda, \sigma}(\boldsymbol{x}_k)\| \to 0$ as $t \to \infty$.
\end{mythm}
Note that in Theorem~\ref{them: 1}, the usual condition on the stepsize becomes a condition on the regularization parameter $\lambda L_f < 1$. The regularization trade-off parameter is then
limited by the value of $L_f$. Even if this is not usual in optimization, we argue that this is a not a problem as the regularization strength is also regulated by the $\sigma$ parameter
which we are free to tune manually.
\begin{mylem}\label{lemma: error_shrink}
    A sufficient condition to achieve the summable of $\{ \epsilon_k \}_{k \geq 1}$ is that $\epsilon_k$ decreases as $O(1/k^{1 + \delta})$ for any $\delta > 0$.
\end{mylem}
When we replace the mismatched MMSE denoiser with the target MMSE denoiser, we recover the traditional PnP-PGD. To highlight the impact of the mismatch, we next provide the same statement but using the target denoiser.

\begin{mythm}\label{them: 2}
    Run PnP-PGD using the MMSE denoiser for $t\geq 1$ iterations under Assumption~\ref{assump: 1}-\ref{assump: 3}, with $\lambda L_f < 1$, then, we have
    \begin{equation}
	\min_{1 \leq k \leq t} \| \nabla F_{\lambda, \sigma} (\boldsymbol{x}_k) \|^2 \leq \frac{1}{t} \sum_{k=1}^t \| \nabla F_{\lambda, \sigma} (\boldsymbol{x}_k) \|^2 \leq \frac{1}{t} \cdot \frac{8}{1 - L_f} \left( F_{\lambda, \sigma}(\boldsymbol{x}_0) - F_{\lambda, \sigma}^* \right).
\end{equation}
\end{mythm}



\subsection{Equivariant PnP-PGD with Mismatched Denoiser}
While prior mismatch is traditionally analyzed as a discrepancy in data distributions (e.g., training on faces vs. testing on landscapes)~\citep{shoushtari2024prior}, we argue that a fundamental source of mismatch stems from the inductive bias of the network architecture itself, which we term \textbf{Structural Prior Mismatch}. Seminal works have established that the architecture of a generator captures low-level image statistics, acting as an implicit prior~\citep{ulyanov2018deep}. However, standard CNN architectures impose rigid geometric constraints, such as fixed grid alignments, that do not strictly align with the continuous transformation groups of natural scenes~\citep{cohen2017inductive}.

This misalignment manifests as spectral bias~\citep{rahaman2019spectral} and aliasing artifacts~\citep{zhang2019making}, representing a systematic deviation $\mathcal{E}$ between the modeled prior $\hat{D}$ and the true prior $D^*$. Unlike stochastic noise, this structural error is often anisotropic~\citep{antun2020instabilities}, creating specific directions of instability in the PnP optimization landscape. Consequently, we employ the Equivariant PnP (EPnP) framework not merely as a heuristic, but as a principled geometric intervention to mitigate this structural mismatch. We demonstrate that EPnP provably reduces the mismatch error term $\epsilon$ defined in Theorem~\ref{them: 1}, providing a concrete mechanism to improve convergence reliability without retraining the prior.

As shown in~\citep{lenc2015understanding}, natural images densities tend to be invariant to some set of transformations such as rotations or flips. To formalize these properties, we define the key notions of invariance and $\pi$-equivariance.

\begin{mydef}[Invariance]
    We denote by $\mathrm{T}_g: \mathbb{R}^n \to \mathbb{R}^n$ a differentiable transformation of $\mathbb{R}^n$ and by $\mathcal{G}$ a measurable set of transformations of $\mathbb{R}^n$. A density $p$ on $\mathbb{R}^n$ is said to be invariant to a set of transformations $\mathcal{G}$ if $\forall g \in \mathcal{G}$, $p = p \circ \mathrm{T}_g$ a.e.
\end{mydef}

\begin{mydef}[$\pi$-equivariance]\label{mydef: 2}
    We denote by $G \sim \pi$ a random variable of law $\pi$ on $\mathcal{G}$. A density $p$ on $\mathbb{R}^n$ is said to be $\pi$-equivariant if $\log p = \mathbb{E}_{G \sim \pi} \left[| \log(p \circ G) | \right]$ and $\mathbb{E}_{G \sim \pi}\left[| \log(p \circ G) | \right] < \infty$. 
\end{mydef}
Definition~\ref{mydef: 2} relaxes the notion of invariance for a density in the following sense. If a density $p$ is invariant to each $g \in \mathcal{G}$, then $p$ is $\pi$-equivariant, whatever the distribution $\pi$ on $\mathcal{G}$.

\begin{myassump}\label{myassup: 6}
    Suppose that the ideal density $p(x)$ on $\mathbb{R}^n$ is invariant to a set of transformations $\mathcal{G}$. All $\mathrm{T}_g \in \mathcal{G}$ are affine transformations with isometric linear parts, i.e. $\forall \mathrm{T}_g \in \mathcal{G}$, there exist $a: \mathbb{R}^n \to \mathbb{R}^n$ a linear isometry and $c \in \mathbb{R}^n$, such that $\forall x \in \mathbb{R}^n, \mathrm{T}_g(x)=ax+c$.
\end{myassump}
Assumption~\ref{myassup: 6} is verified for a set of transformations as detailed in~\cite{renaud2025equivariant}.

\begin{myassump}\label{myassup: 7}
    Suppose that the Gaussian noise $\boldsymbol{e} \sim \mathcal{N}(\mathbf{0}, \sigma^2 \mathbf{I})$ is isotropic.
\end{myassump}
This assumption assumes the measurement noise follows an isotropic distribution (e.g., Gaussian white noise). This is a standard assumption in the PnP and RED literature~\citep{shoushtari2024prior, shoushtari2022deep} and aligns with the physical properties of generic sensor noise. We note that for specific structured noise types (e.g., rain streaks or banding), this assumption may be relaxed, but such cases are outside the scope of this theoretical analysis.

\begin{myprop}\label{myprop: 2}
    With Assumption~\ref{myassup: 6},~\ref{myassup: 7}, the target denoiser $\mathrm{D}^*_\sigma$ is $\pi$-equivairant, which means for $\forall \mathrm{T}_g \in \mathcal{G}, x \in \mathbb{R}^n$:
    \begin{equation}
        \mathrm{T}_g^{-1} \circ \mathrm{D}^*_\sigma (\mathrm{T}_g \circ x) = \mathrm{D}^*_\sigma (x).
    \end{equation}
\end{myprop}
With Assumption~\ref{assump: 5} that the mismatched denoiser $\hat{\mathrm{D}}_\sigma (z)$ exists bias, here we denote as
\[
    \hat{\mathrm{D}}_\sigma (z) = \mathrm{D}^*_\sigma (z) + \mathcal{E}(z),
\]
where $\| \mathcal{E}(z) \|^2 \propto \epsilon$. We then denote the equivairant denoiser $\tilde{\mathrm{D}}_\sigma$ as
\[
    \tilde{\mathrm{D}}_\sigma(z) := \mathbb{E}_{g \sim \pi} \left[ \mathrm{T}_g^{T} \circ \hat{\mathrm{D}}_\sigma (\mathrm{T}_g \circ z)  \right] = \int_G \mathrm{T}_g \cdot \hat{\mathrm{D}}_\sigma(\mathrm{T}_g z)d\pi(g),
\]
where $\pi$ is the normalized Haar measure defined on group $G$.

Before present our main result, we first present some definitions and a useful lemma as follows.
\begin{mydef}[Isotropic Bias]
    The bias $\mathcal{E}$ satisfies $\mathcal{E}(\mathrm{T}_g x) = \mathrm{T}_g \mathcal{E}(x)$ for $\forall \mathrm{T}_g \in G$ a.e. 
\end{mydef}

\begin{mydef}[Anisotropic Bias]
    The vector value function $V_x(\mathrm{T}_g) : = \mathrm{T}_g^{-1} \mathcal{E}(\mathrm{T}_gx)$ is non-constant for $\forall T_g \in G$ a.e. In other word, there exists a positive measurable set satisfying $V_x (\mathrm{T}_g) \neq \mathrm{const}$.
\end{mydef}

\begin{mylem}[Equivariant Bias Reduction]\label{mylem: EBR} 
    For $\forall x \in \mathcal{X}$ and $g \sim G$, if the bias $\mathcal{E}(x)$ generated by PnP is anisotropic, then the bias $ \tilde{\mathcal{E}}(x) $ generated by EPnP satisfies
    \begin{equation}
        \| \tilde{\mathcal{E}}(x) \|^2 < \int_G \| \mathcal{E}(\mathrm{T}_g x) \|^2 d\pi(g) = \mathbb{E}_{g} [\|\mathcal{E}(\mathrm{T}_g x)\|^2].
    \end{equation}
\end{mylem}

\begin{myremark}
    We denote $\tilde{\epsilon}_k$ as the error of EPnP:
    \begin{equation}
	\| \tilde{\mathrm{D}}_\sigma (z_{k+1}) - \mathrm{D}_\sigma (z_{k+1}) \| \leq \sqrt{2(L+1)\tilde{\epsilon}_{k+1}},
    \end{equation}
    and
    \begin{equation}
	\mathbb{E}_{g \sim \pi} \left[ \| \hat{\mathrm{D}}_\sigma (\mathrm{T}_g  z_{k+1}) - \mathrm{D}_\sigma (\mathrm{T}_g z_{k+1}) \| \right] \leq \sqrt{2(L+1)\hat{\epsilon}_{k+1}} .
    \end{equation}
    the group averaged PnP error. Then we can conclude that 
    \[
    \tilde{\epsilon}_k = \hat{\epsilon}_k - \text{Var}_{g}  < \hat{\epsilon}_k.
    \]
\end{myremark}
Now we present our main theoretical result.
\begin{mythm}\label{mythm: 3}
    Run EPnP-PGD using a $\textbf{equivariant mismatched}$ MMSE denoiser for $t \geq 1$ iterations under Assumption~\ref{assump: 1}-\ref{assump: 5} with $\lambda L_f < 1$, then we have
    \begin{equation}
	\min_{1 \leq k \leq t} \| \nabla F_{\lambda, \sigma} (x_k) \|^2 \leq \frac{1}{t} \sum_{k+1}^t \| \nabla F_{\lambda, \sigma} (x_k) \|^2 \leq \frac{1}{t} \cdot \frac{16}{1 - L_f} \left( F_{\lambda, \sigma}(x_0) - F_{\lambda, \sigma}^* \right) + \frac{1}{t} \cdot \left( \frac{16}{1- L_f} + \frac{4}{1-L} \right) \sum_{k=1}^t \tilde{\epsilon}_k.
\end{equation}
Specifically, with Assumption~\ref{myassup: 6} - \ref{myassup: 7}, we have following convergence result for EPnP-PGD:
\begin{equation}
    \sum_{k=1}^t \tilde{\epsilon}_k =  \sum_{k=1}^t \left( \hat{\epsilon}_k - \underbrace{\text{Var}_{g, k}}_{\text{Gain}} \right) <  \hat{\epsilon}_k,
\end{equation}
where $\text{Var}_{g, k} =\text{Var}_{g,k}(\mathrm{T}_g^{-1} \cdot \mathcal{E}(\mathrm{T}_g z_k)) > 0$, the more isotropic the bias becomes, the more robust the convergence trajectory.
\end{mythm}




\section{Proofs}
\subsection{Warm-up}

We prove first $\frac{1}{L+1}$-Strong-Convexity of $H(\boldsymbol{x})$
\begin{proof}
    According to Proposition~\ref{prop: 1}: the denoiser $\mathsf{D}_\sigma = \text{I}d - \nabla g_\sigma$, where $\nabla g_\sigma$ is $L$-Lipschitz and $L<1$, and $\phi_\sigma$ is $\frac{L}{L+1}$ semi-convex, which means $\phi_\sigma (\boldsymbol{x}) + \frac{1}{2} \cdot \frac{L}{L+1} \| \boldsymbol{x}\|^2$ is convex.\\
    Then we can rewrite $H(\boldsymbol{x})$ with
    \begin{equation}
    	H(\boldsymbol{x}) = \left[ \phi_\sigma(\boldsymbol{x}) + \frac{1}{2} \cdot \frac{L}{L+1} \|\boldsymbol{x}\|^2  \right] + \left[ \frac{1}{2} \| \boldsymbol{x} - \boldsymbol{z}_{k+1}\|^2 - \frac{1}{2} \cdot \frac{L}{L+1} \|\boldsymbol{x}\|^2 \right],
\end{equation}
where the first term is convex, and remain quadratic term's Hessian is $(1-\frac{L}{L+1})I=\frac{1}{L+1}I \succ 0$, thus the second term is also convex. In total, $H(\boldsymbol{x})$ is $\frac{1}{L+1}$ strongly convex, this complete the proof.
\end{proof}

\subsection{Proof of Lemma~\ref{lemma: error_shrink}}
\begin{proof}
    If there exists $C>0, \delta >0$, satisfy
    \[
    \epsilon_k \leq \frac{C}{k^{1+\delta}},
    \]
    then, since $1+\delta>1$, we have
    \[
    \sum_{k=1}^{\infty} \epsilon_k \leq C \sum_{k=1}^{\infty}\frac{1}{k^{1+\delta}}
    \]
    is convergence, which complete the proof.
\end{proof}

\subsection{Proof of Theorem~\ref{them: 1}}
Before proof the main theorem, we first present some useful lemmas.
\begin{mylem}\label{lemma: 1}
    Assume that Assumption~\ref{assump: 1}-\ref{assump: 5} hold and let the sequence $\{ \boldsymbol{x}_k \}$ be generated via iterations of PnP-PGD with the $\textbf{mismatched}$ MMSE denoiser and $\lambda L_f <1$. Then for the objective function $F_{\lambda,\sigma}$ defined above, we have
    \begin{equation}
        F_{\lambda, \sigma}(\boldsymbol{x}_{k+1}) \leq F_{\lambda, \sigma} (\boldsymbol{x}_k)  - \frac{1-L_f}{2} \| \boldsymbol{x}_{k+1} - \boldsymbol{x}_k \|^2 + \epsilon_{k+1}.
    \end{equation}
\end{mylem}
\begin{proof}
    Under Assumption~\ref{assump: 5}, denoiser $\widehat{\mathsf{D}}_\sigma$ is an approximation of target MMSE denoiser, at $k+1$-th iteration, suppose that $\boldsymbol{x}_{k+1}= \widehat{\mathsf{D}}_\sigma (\boldsymbol{z}_{k+1})$, we have
    \[
    \phi_\sigma(\boldsymbol{x}_{k+1}) + \frac{1}{2} \| \boldsymbol{x}_{k+1} - \boldsymbol{x}_k \| ^2 \leq \min_{\boldsymbol{u}} \left( \phi_\sigma(\boldsymbol{u}) + \frac{1}{2} \| \boldsymbol{u} - \boldsymbol{z}_{k+1} \|^2 \right) + \epsilon_{k+1},
    \]
    Introducing the majorization-minimization (MM) function
    \[
    Q (\boldsymbol{x}, \boldsymbol{y}) = \lambda f(\boldsymbol{y}) + \lambda \left< \boldsymbol{x}-\boldsymbol{y}, \nabla f(\boldsymbol{y}) \right> + \frac{1}{2} \|\boldsymbol{x}-\boldsymbol{y}\|^2 + \phi_\sigma(\boldsymbol{x}),
    \]
    satisfies
    \[
    Q(\boldsymbol{x}_{k+1}, \boldsymbol{x}_k) \leq \min_{\boldsymbol{x}} Q(\boldsymbol{x}, \boldsymbol{x}_k) + \epsilon_{k+1},
    \]
    since 
    \[
    \min_{\boldsymbol{x}} Q(\boldsymbol{x}, \boldsymbol{x}_k) \leq Q(\boldsymbol{x}_k, \boldsymbol{x}_k) = F_{\lambda, \sigma}(\boldsymbol{x}_k),
    \]
    we have
    \[
    Q(\boldsymbol{x}_{k+1}, \boldsymbol{x}_k) \leq Q(\boldsymbol{x}_k, \boldsymbol{x}_k) + \epsilon_{k+1}.
    \]
    Then, note that $Q(\boldsymbol{x}_{k+1}, \boldsymbol{x}_k) \leq Q(\boldsymbol{x}_k, \boldsymbol{x}_k) + \epsilon_{k+1}$ implies that
    \begin{align}
        \lambda f(\boldsymbol{x}_k) + \lambda \left< \boldsymbol{x}_{k+1}-\boldsymbol{x}_k, \nabla f(\boldsymbol{x}_k) \right> + \frac{1}{2} \| \boldsymbol{x}_{k+1} - \boldsymbol{x}_k \|^2 + \phi_\sigma(\boldsymbol{x}_{k+1}) &\leq \lambda f(\boldsymbol{x}_k) + \phi_\sigma(\boldsymbol{x}_k) + \epsilon_{k+1} \\
        \label{eq: mm}
         \phi_\sigma(\boldsymbol{x}_{k+1}) \leq \phi_\sigma(\boldsymbol{x}_k) - \lambda \left< \boldsymbol{x}_{k+1}-\boldsymbol{x}_k, \nabla f(\boldsymbol{x}_k) \right> &- \frac{1}{2} \| \boldsymbol{x}_{k+1} - \boldsymbol{x}_k \|^2 + \epsilon_{k+1}
    \end{align}
    using the descent lemma
    \[
    f(\boldsymbol{x}) \leq f(\boldsymbol{y}) + \left< \boldsymbol{x}-\boldsymbol{y}, \nabla f(\boldsymbol{y}) \right> + \frac{1}{2t} \|\boldsymbol{x}-\boldsymbol{y}\|^2
    \]
    with $t = \frac{1}{L_f}$ and relation~\eqref{eq: mm} we get
    \begin{align*}
	F_{\lambda, \sigma} (\boldsymbol{x}_{k+1}) &= \lambda f(\boldsymbol{x}_{k+1}) + \phi_\sigma(\boldsymbol{x}_{k+1}) \\
	&\leq \phi_\sigma(\boldsymbol{x}_k) - \lambda \left< \boldsymbol{x}_{k+1} - \boldsymbol{x}_k, \nabla f(\boldsymbol{x}_k) \right> - \frac{1}{2} \| \boldsymbol{x}_{k+1} - \boldsymbol{x}_k \|^2 + \epsilon_{k+1} + \lambda f(\boldsymbol{x}_k) + \\ &\qquad \lambda \left< \boldsymbol{x}_{k+1} - \boldsymbol{x}_k, \nabla f(\boldsymbol{x}_k) \right> + \frac{L_f}{2} \|\boldsymbol{x}_{k+1} - \boldsymbol{x}_k \|^2 \\
	&= F_{\lambda, \sigma} (\boldsymbol{x}_k) + \epsilon_{k+1} - \frac{1-L_f}{2} \|\boldsymbol{x}_{k+1} - \boldsymbol{x}_k \|^2,
    \end{align*}
    this complete the proof.

\end{proof}

\begin{mylem}\label{lemma: 2}
    Assume that Assumption~\ref{assump: 1}-\ref{assump: 5} holds, then for the $\epsilon$-optimal solution $\widehat{\mathsf{D}}_\sigma (\boldsymbol{z}) = \hat{\boldsymbol{x}}$ to the proximal problem:
    \begin{equation}
        H(\boldsymbol{u}) = \mathop{\arg\min}_{\boldsymbol{u}} \frac{1}{2} \| \boldsymbol{u} - \boldsymbol{z} \|^2 + \phi_\sigma(\boldsymbol{u}),
    \end{equation}
    we have
    \begin{equation}
	\begin{cases}
		\nabla \phi_\sigma(\hat{\boldsymbol{x}}) = \boldsymbol{z} - \hat{\boldsymbol{x}} - \boldsymbol{\delta} \\
		\| \boldsymbol{\delta} \| \leq \sqrt{\frac{2\epsilon}{1-L}}.
	\end{cases}
\end{equation}
\end{mylem}
\begin{proof}
     For $\nabla^2 H(\boldsymbol{x}) = I + \nabla^2 \phi_\sigma(\boldsymbol{x})$, with Proposition~\ref{prop: 1}, we have $\mathsf{D}_\sigma^{-1}(\boldsymbol{x}) = \boldsymbol{x} + \nabla \phi_\sigma (\boldsymbol{x})$. For mapping $\boldsymbol{y} = \mathsf{D}_\sigma^{-1}(\boldsymbol{x})$, the Jacobian is $J_{\mathsf{D}_\sigma^{-1}(\boldsymbol{x})} = I + \nabla^2 \phi_\sigma(\boldsymbol{x})$, by inverse theorem we can get $J_{\mathsf{D}_\sigma^{-1}(\boldsymbol{x})} = \left[ J_{\mathsf{D}_\sigma(\boldsymbol{y})} \right]^{-1}$, so 
     \[
     \nabla^2 H(\boldsymbol{x}) = I + \nabla^2 \phi_\sigma (\boldsymbol{x}) = \left[ J_{D_\sigma(\boldsymbol{y})} \right]^{-1}.
     \]
     Then, we know that $\mathsf{D}_\sigma = \nabla h_\sigma$, as a result $J_{\mathsf{D}_\sigma(\boldsymbol{y})} = \nabla^2 h_\sigma (\boldsymbol{y})$.\\
     Since $h_\sigma$ is $(1-L)$-strongly convex, which implies
     \[
     \nabla^2 h_\sigma (\boldsymbol{y}) \succeq (1-L) I, \quad (L<1)
     \]
     then we have
     \[
     \lambda_{\max} (\nabla^2 H(\boldsymbol{x})) = \frac{1}{\lambda_{\min}J_{\mathsf{D}_\sigma (\boldsymbol{y})}} \leq \frac{1}{1-L}
     \]
     which means $H(\boldsymbol{x})$ is $\frac{1}{1-L}$-smooth. Since $H(\boldsymbol{x})$ is differentiable, denote that
     \[
     \boldsymbol{\delta} \triangleq - \nabla H(\hat{\boldsymbol{x}}) = \boldsymbol{z} - \hat{\boldsymbol{x}} - \nabla \phi_\sigma(\hat{\boldsymbol{x}}),
     \]
     then we get
     \[
     \nabla \phi_\sigma (\hat{\boldsymbol{x}}) = \boldsymbol{z} - \hat{\boldsymbol{x}} - \boldsymbol{\delta},
     \]
     since $\hat{\boldsymbol{x}}$ is an $\epsilon$-optimal solution to $H(\boldsymbol{u})$, we have
     \[
     H(\hat{\boldsymbol{x}}) - H(\boldsymbol{x}^*) \leq \epsilon.
     \]
     Using the $\frac{1}{1-L}$-smoothness of $H(\boldsymbol{x})$, we have
     \[
     \| \nabla H(\hat{\boldsymbol{x}}) \|^2 \leq \frac{2}{1-L} \left( H(\hat{\boldsymbol{x}}) - H(\boldsymbol{x}^*) \right) \leq \frac{2\epsilon}{1-L},
     \]
     thus,
     \[
     \| \boldsymbol{\delta} \|^2 \leq \frac{2\epsilon}{1-L} \quad \text{and}, \quad \nabla \phi_\sigma(\hat{\boldsymbol{x}}) = \boldsymbol{z} - \hat{\boldsymbol{x}} - \boldsymbol{\delta}.
     \]
     this complete proof.
\end{proof}
$\textbf{Theorem}$ Run PnP-PGD using a $\textbf{mismatched}$ MMSE denoiser for $t \geq 1$ iterations under Assumption~\ref{assump: 1}-\ref{assump: 5} with $\lambda L_f < 1$, then we have
\begin{equation}
	\min_{1 \leq k \leq t} \| \nabla F_{\lambda, \sigma} (\boldsymbol{x}_k) \|^2 \leq \frac{1}{t} \sum_{k=1}^t \| \nabla F_{\lambda, \sigma} (\boldsymbol{x}_k) \|^2 \leq \frac{1}{t} \cdot \frac{16}{1 - L_f} \left( F_{\lambda, \sigma}(\boldsymbol{x}_0) - F_{\lambda, \sigma}^* \right) + \frac{1}{t} \cdot \left( \frac{16}{1- L_f} + \frac{4}{1-L} \right) \sum_{k=1}^t \epsilon_k.
\end{equation}

\begin{proof}
    At $k+1$-th iteration, for the output of mismatched denoiser $\boldsymbol{x}_{k+1} = \widehat{\mathsf{D}}_\sigma (\boldsymbol{z}_{k+1})$, by Lemma~\ref{lemma: 2}, we have
    \begin{align*}
	\nabla F_{\lambda, \sigma} (\boldsymbol{x}_{k+1}) &= \lambda \nabla f(\boldsymbol{x}_{k+1}) + \nabla \phi_\sigma (\boldsymbol{x}_{k+1}) \\
	&= \lambda \nabla f(\boldsymbol{x}_{k+1}) + \boldsymbol{z}_{k+1} - \boldsymbol{x}_{k+1} - \boldsymbol{\delta}_{k+1} \\
	&= \lambda \nabla f(\boldsymbol{x}_{k+1}) + \boldsymbol{x}_k - \lambda \nabla f(\boldsymbol{x}_k) - \boldsymbol{x}_{k+1} - \boldsymbol{\delta}_{k+1} \\
	&= (\boldsymbol{x}_k - \boldsymbol{x}_{k+1}) + \lambda \left( \nabla f(\boldsymbol{x}_{k+1}) - \nabla f(\boldsymbol{x}_k) \right) - \boldsymbol{\delta}_{k+1}, \\
    \end{align*}
    thus
    \[
    \nabla F_{\lambda, \sigma} (\boldsymbol{x}_{k+1}) \leq \| \boldsymbol{x}_k - \boldsymbol{x}_{k+1} \| + \lambda \| \nabla f(\boldsymbol{x}_{k+1}) - \nabla f(\boldsymbol{x}_k) \| + \| \boldsymbol{\delta}_{k+1} \|,
    \]
    since $\nabla f$ is $L_f$-Lipschitz, we have
    \[
    \| \nabla F_{\sigma, \lambda} (\boldsymbol{x}_{k+1}) \| \leq (1+\lambda L_f) \| \boldsymbol{x}_{k+1} - \boldsymbol{x}_k \| + \| \boldsymbol{\delta}_{k+1} \|,
    \]
    using $(\mathbf{a} + \mathbf{b})^2 \leq 2 \mathbf{a}^2 + 2 \mathbf{b}^2$ and Lemma~\ref{lemma: 2}, we have
    \[
    \| \nabla F_{\lambda, \sigma} (\boldsymbol{x}_{k+1}) \|^2 \leq 2 (1+\lambda L_f)^2 \| \boldsymbol{x}_{k+1} -\boldsymbol{x}_k \|^2 + 4 \frac{\epsilon_{k+1}}{1-L},
    \]
    using Lemma~\ref{lemma: 1}, we have
    \[
    \| \boldsymbol{x}_{k+1} - \boldsymbol{x}_k \|^2 \leq \frac{2}{1 - L_f} \left( F_{\lambda, \sigma}(\boldsymbol{x}_k) - F_{\lambda, \sigma}(\boldsymbol{x}_{k+1}) \right) + \frac{2\epsilon_{k+1}}{1-L_f},
    \]
    combining these two relations,
    \begin{align*}
        \| \nabla F_{\lambda, \sigma} (\boldsymbol{x}_{k+1}) \|^2 &\leq \frac{4(1+\lambda L_f)^2}{1-L_f} \left( F_{\lambda, \sigma} (\boldsymbol{x}_{k}) - F_{\lambda, \sigma} (\boldsymbol{x}_{k+1}) \right) + \frac{4(1+\lambda L_f)^2}{1-L_f} \epsilon_{k+1} + \frac{4 \epsilon_{k+1}}{1-L} \\
        &\leq \frac{16}{1-L_f} \left( F_{\lambda, \sigma} (\boldsymbol{x}_{k}) - F_{\lambda, \sigma} (\boldsymbol{x}_{k+1}) \right) + \left( \frac{16}{1-L_f} + \frac{4}{1-L} \right) \epsilon_{k+1},
    \end{align*}
    summing over $k=0, 1, \ldots, t-1$ gives
    \[
    \sum_{k=1}^t \| \nabla F_{\lambda, \sigma} (\boldsymbol{x}_k) \|^2 \leq \frac{16}{1-L_f} \left( F_{\lambda, \sigma} (\boldsymbol{x}_{0}) - F_{\lambda, \sigma}^* \right) + \left( \frac{16}{1-L_f} + \frac{4}{1-L} \right) \sum_{k=1}^t \epsilon_{k}
    \]
    Then by averaging over $t \geq 1$ iterations, we have
    \[
    \min_{1 \leq k \leq t} \| \nabla F_{\lambda, \sigma} (\boldsymbol{x}_k) \|^2 \leq \frac{1}{t}\sum_{k=1}^t \| \nabla F_{\lambda, \sigma} (\boldsymbol{x}_k) \|^2 \leq \frac{1}{t} \cdot \frac{16}{1-L_f} \left( F_{\lambda, \sigma} (\boldsymbol{x}_{0}) - F_{\lambda, \sigma}^* \right) + \left( \frac{16}{1-L_f} + \frac{4}{1-L} \right)\frac{1}{t} \cdot \sum_{k=1}^t \epsilon_{k},
    \]
    this complete the proof.
\end{proof}

\subsection{Proof of Theorem~\ref{them: 2}}
Before the proof of main theorem, we first present a useful lemma.
\begin{mylem}\label{lemma: 3}
    Assume that Assumption~\ref{assump: 1}-\ref{assump: 3} hold and let the sequence $\{ \boldsymbol{x}_k \}$ be generated via iterations of PnP-PGD using the MMSE denoiser with $\lambda L_f <1$. Then for the objective function $F_{\lambda,\sigma}$ defined above, we have
    \begin{equation}
        F_{\lambda, \sigma}(\boldsymbol{x}_{k+1}) \leq F_{\lambda, \sigma} (\boldsymbol{x}_k) - \frac{1-L_f}{2} \| \boldsymbol{x}_{k+1} - \boldsymbol{x}_k \|^2.
    \end{equation}
\end{mylem}
\begin{proof}
    Introducing the majorization-minimization (MM) function
    \[
    Q (\boldsymbol{x}, \boldsymbol{y}) = \lambda f(\boldsymbol{y}) + \lambda \left< \boldsymbol{x}-\boldsymbol{y}, \nabla f(\boldsymbol{y}) \right> + \frac{1}{2} \|\boldsymbol{x}-\boldsymbol{y}\|^2 + \phi_\sigma(\boldsymbol{x}),
    \]
    we have
    \[
    Q(\boldsymbol{x},\boldsymbol{x}) = F_{\lambda, \sigma}(\boldsymbol{x}),
    \]
    and
    \begin{align*}
	\mathop{\arg\min}_{\boldsymbol{x}} Q(\boldsymbol{x},\boldsymbol{y}) &= \mathop{\arg\min}_{\boldsymbol{x}} \lambda \left< \boldsymbol{x}-\boldsymbol{y}, \nabla f(\boldsymbol{y}) \right> + \frac{1}{2} \|\boldsymbol{x}-\boldsymbol{y} \|^2 + \phi_\sigma(\boldsymbol{x}) \\
	&= \mathop{\arg\min}_{\boldsymbol{x}} \phi_\sigma(\boldsymbol{x}) + \frac{1}{2} \| \boldsymbol{x} - (\boldsymbol{y} - \lambda \nabla f(\boldsymbol{y})) \|^2 \\
	&= \mathrm{Prox}_{\phi_\sigma} \circ (Id - \lambda \nabla f) (\boldsymbol{y}),
    \end{align*}
    the $\mathop{\arg\min}$ is unique by Proposition~\ref{prop: 1} and by definition of the $\mathop{\arg\min}$, $\boldsymbol{x}_{k+1} = \mathrm{Prox}_{\phi_\sigma} \circ (Id - \lambda \nabla f) (\boldsymbol{x}_k)$ implies that
    \[
    Q(\boldsymbol{x}_{k+1}, \boldsymbol{x}_k) \leq Q(\boldsymbol{x}_k, \boldsymbol{x}_k) = F_{\lambda, \sigma}(\boldsymbol{x}_k).
    \]
    Moreover, with $f$ being $L_f$-smooth, we have by the descent lemma, for any $t \leq \frac{1}{L_f}$ and any $\boldsymbol{x}, \boldsymbol{y}\in \mathbb{R}^n$,
    \[
    f(\boldsymbol{x}) \leq f(\boldsymbol{y}) + \left< \boldsymbol{x}-\boldsymbol{y}, \nabla f(\boldsymbol{y}) \right> + \frac{1}{2t} \| \boldsymbol{x}-\boldsymbol{y} \|^2.
    \]
    Note that $Q(\boldsymbol{x}_{k+1}, \boldsymbol{x}_k) \leq Q(\boldsymbol{x}_k, \boldsymbol{x}_k)$ implies that 
    \[
    \phi_\sigma (\boldsymbol{x}_{k+1}) \leq \phi_\sigma(\boldsymbol{x}_k) - \lambda \left< \boldsymbol{x}_{k+1}-\boldsymbol{x}_k, \nabla f(\boldsymbol{x}_k) \right> - \frac{1}{2} \| \boldsymbol{x}_{k+1} - \boldsymbol{x}_k \|^2,
    \]
    using above descent lemma with stepsize $t = \frac{1}{L_f}$, we get
    \begin{align*}
	F_{\lambda, \sigma} (\boldsymbol{x}_{k+1}) &= \lambda f(\boldsymbol{x}_{k+1}) + \phi_\sigma(\boldsymbol{x}_{k+1}) \\
	&\leq \phi_\sigma(\boldsymbol{x}_k) - \lambda \left< \boldsymbol{x}_{k+1} - \boldsymbol{x}_k, \nabla f(\boldsymbol{x}_k) \right> - \frac{1}{2} \| \boldsymbol{x}_{k+1} - \boldsymbol{x}_k \|^2 + \lambda f(\boldsymbol{x}_k) + \lambda \left< \boldsymbol{x}_{k+1} - \boldsymbol{x}_k, \nabla f(\boldsymbol{x}_k) \right> \\ &\quad + \frac{L_f}{2} \|\boldsymbol{x}_{k+1} - \boldsymbol{x}_k \|^2 \\
	&= F_{\lambda, \sigma} (\boldsymbol{x}_k) - \frac{1-L_f}{2} \|\boldsymbol{x}_{k+1} - \boldsymbol{x}_k \|^2,
\end{align*}
this complete the proof.
\end{proof}
\textbf{Theorem} Run PnP-PGD using the MMSE denoiser for $t\geq 1$ iterations under Assumption~\ref{assump: 1}-\ref{assump: 3}, with $\lambda L_f < 1$, then, we have
    \begin{equation}
	\min_{1 \leq k \leq t} \| \nabla F_{\lambda, \sigma} (\boldsymbol{x}_k) \|^2 \leq \frac{1}{t} \sum_{k=1}^t \| \nabla F_{\lambda, \sigma} (\boldsymbol{x}_k) \|^2 \leq \frac{1}{t} \cdot \frac{8}{1 - L_f} \left( F_{\lambda, \sigma}(\boldsymbol{x}_0) - F_{\lambda, \sigma}^* \right).
    \end{equation}

\begin{proof}
    By Proposition~\ref{prop: 1}, for $\boldsymbol{x} \in \mathsf{Im}(\mathsf{D}_\sigma)$, $\phi_\sigma$ is differentiable,  and
    \[
    \nabla \phi_\sigma(\boldsymbol{x}) = \mathsf{D}^{-1}_\sigma(\boldsymbol{x}) - \boldsymbol{x},
    \]
    for the PnP-PGD iterates:
    \begin{equation*}
        \begin{cases}
        \boldsymbol{z}_{k+1} = \boldsymbol{x}_k - \lambda \nabla f(\boldsymbol{x}_k) \\
        \boldsymbol{x}_{k+1} = \mathsf{D}_\sigma(\boldsymbol{z}_{k+1})
        \end{cases}
    \end{equation*}
    we first derivate the optimality condition:
    \begin{align*}
	\nabla \phi_\sigma (\boldsymbol{x}_{k+1}) &= \mathsf{D}_\sigma^{-1} (\boldsymbol{x}_{k+1}) - \boldsymbol{x}_{k+1} \\ 
	&= \boldsymbol{z}_{k+1} - \boldsymbol{x}_{k+1} \\
	&= \boldsymbol{x}_k - \lambda \nabla f(\boldsymbol{x}_k) - \boldsymbol{x}_{k+1}
    \end{align*}
    and $\nabla F_{\lambda, \sigma} (\boldsymbol{x}_{k+1}) = \lambda \nabla f(\boldsymbol{x}_{k+1}) + \nabla \phi_\sigma (\boldsymbol{x}_{k+1})$, we then have
    \begin{align*}
        \nabla F_{\lambda, \sigma} (\boldsymbol{x}_{k+1}) &= \lambda \nabla f(\boldsymbol{x}_{k+1}) + \nabla \phi_\sigma (\boldsymbol{x}_{k+1}) \\
        &= \lambda \nabla f(\boldsymbol{x}_{k+1}) + \boldsymbol{x}_k - \lambda \nabla f(\boldsymbol{x}_k) - \boldsymbol{x}_{k+1} \\
        &= (\boldsymbol{x}_k - \boldsymbol{x}_{k+1}) + \lambda \left(\nabla f(\boldsymbol{x}_{k+1}) - \nabla f(\boldsymbol{x}_k)\right)
    \end{align*}
    thus
    \begin{align*}
	\| \nabla F_{\lambda, \sigma} (\boldsymbol{x}_{k+1}) \| &\leq \| \boldsymbol{x}_k - \boldsymbol{x}_{k+1} \| + \lambda \| \nabla f(\boldsymbol{x}_{k+1}) - \nabla f(\boldsymbol{x}_k) \| \\
    & \leq  \| \boldsymbol{x}_k - \boldsymbol{x}_{k+1} \| + \lambda L_f \| \boldsymbol{x}_{k+1} -  \boldsymbol{x}_k \| \\
    &\leq (1+\lambda L_f) \| \boldsymbol{x}_{k+1} - \boldsymbol{x}_k \|
    \end{align*}
    then we have
    \[
    \| \nabla F_{\lambda, \sigma} (\boldsymbol{x}_{k+1}) \|^2 \leq (1+\lambda L_f)^2 \| \boldsymbol{x}_{k+1} - \boldsymbol{x}_k \|^2.
    \]
    By Lemma~\ref{lemma: 3}, we have
    \[
    \| \boldsymbol{x}_{k+1} - \boldsymbol{x}_k \|^2 \leq \frac{2}{1-L_f} \left( F_{\lambda, \sigma}(\boldsymbol{x}_k) - F_{\lambda, \sigma}(\boldsymbol{x}_{k+1}) \right),
    \]
    combining these two relations, we get
    \begin{align*}
        \| \nabla F_{\lambda, \sigma} (\boldsymbol{x}_{k+1}) \|^2 &\leq \frac{2(1+\lambda L_f)^2}{1-L_f} \left( F_{\lambda, \sigma}(\boldsymbol{x}_k) - F_{\lambda, \sigma}(\boldsymbol{x}_{k+1}) \right) \\
        &\leq \frac{8}{1-L_f} \left( F_{\lambda, \sigma}(\boldsymbol{x}_k) - F_{\lambda, \sigma}(\boldsymbol{x}_{k+1}) \right),
    \end{align*}
    summing over $k=1, 2, \ldots, t$ gives
    \[
    \sum_{k=1}^t \| \nabla F_{\lambda, \sigma} (\boldsymbol{x}_k) \|^2 \leq \frac{8}{1 - L_f} \left( F_{\lambda, \sigma}(\boldsymbol{x}_0) - F_{\lambda, \sigma}^* \right).
    \]
    Then by averaging over $t \geq 1$ iterations, we have
    \[
    \min_{1 \leq k \leq t} \| \nabla F_{\lambda, \sigma} (\boldsymbol{x}_k) \|^2 \leq \frac{1}{t} \sum_{k=1}^t \| \nabla F_{\lambda, \sigma} (\boldsymbol{x}_k) \|^2 \leq \frac{1}{t} \cdot \frac{8}{1 - L_f} \left( F_{\lambda, \sigma}(\boldsymbol{x}_0) - F_{\lambda, \sigma}^* \right),
    \]
    this complete the proof.
\end{proof}

\subsection{Proof of Proposition~\ref{myprop: 2}}
\textbf{Proposition} With Assumption~\ref{myassup: 6},~\ref{myassup: 7}, the target denoiser $\mathrm{D}^*_\sigma$ is $\pi$-equivairant, which means for $\forall \mathrm{T}_g \in \mathcal{G}, x \in \mathbb{R}^n$:
    \begin{equation}
        \mathrm{T}_g^{-1} \circ \mathrm{D}^*_\sigma (\mathrm{T}_g \circ x) = \mathrm{D}^*_\sigma (x).
    \end{equation}

\begin{proof}
    We denote the clean image distribution as $p(x)$, noised image distribution as $p_\sigma(y)$:
    \begin{equation*}
        p_\sigma (y) = \int p(x) \phi(y-x)dx,
    \end{equation*}
    for the transformed noised distribution:
    \[
    p_\sigma (\mathrm{T}_g y) = \int p(x) \phi(\mathrm{T}_gy-x)dx,
    \]
    let $x' = \mathrm{T}_g^T x$, i.e. $x = \mathrm{T}_g x'$. Since $\mathrm{T}_g$ is orthogonal, and $J |\det(\mathrm{T}_g)|=1$, thus, $dx = dx'$. Plug this into above equation, yields
    \[
    p_\sigma(\mathrm{T}_g y) = \int p(\mathrm{T}_g x') \phi(\mathrm{T}_g y - \mathrm{T}_g x')dx'.
    \]
    since $p(\mathrm{T}_g x') = p(x')$ and $\phi(\mathrm{T}_g(y-x')) = \phi(y-x')$ from Assumption~\ref{myassup: 6},~\ref{myassup: 7}, we have
    \[
    p_\sigma(\mathrm{T}_g y) = \int p(x') \phi(y-x')dx' 
    \]
    which we can conclude that
    \[
    p_\sigma(\mathrm{T}_g y) = p_\sigma(y),
    \]
    which means the noised distribution is also invariant.
    Differentiating both sides w.r.t $y$ of above equation
    \[
    \text{RHS} = \nabla_y p_\sigma(y)
    \]
    \[
    \text{LHS} = \nabla_y p_\sigma (\mathrm{T}_g y) = \mathrm{T}_g^T \nabla p_\sigma(\mathrm{T}_g y).
    \]
    Then we can have
    \[
    \nabla_y p_\sigma(y) = \mathrm{T}_g^T \nabla p_\sigma (\mathrm{T}_g y),
    \]
    which is equivalent to
    \begin{align*}
        \mathrm{T}_g \nabla_y p_\sigma(y) &= \mathrm{T}_g \mathrm{T}_g^T \nabla p_\sigma (\mathrm{T}_g y) \\
        &= \nabla p_\sigma(\mathrm{T}_g y).
    \end{align*}
    Dividing both side by $p_\sigma(\mathrm{T}_g y)$:
    \[
    \frac{\mathrm{T}_g \nabla_y p_\sigma (y)}{p_\sigma(y)} = \frac{\nabla p_\sigma(\mathrm{T}_g y)}{p_\sigma(\mathrm{T}_g y)}
    \]
    which means
    \begin{equation}\label{eq: score}
    \nabla \log p_\sigma(\mathrm{T}_g y) = \mathrm{T}_g \nabla \log p_\sigma (y),
    \end{equation}
    then by Tweedie formula, the target MMSE denoiser is defined as:
    \[
    \mathrm{D}_\sigma^*(y) = y + \sigma^2 \nabla \log p_\sigma (y)
    \]
    and
    \[
    \mathrm{D}_\sigma^* (\mathrm{T}_g y) = \mathrm{T}_g y + \sigma^2 \nabla \log p_\sigma(\mathrm{T}_g y).
    \]
    Using Eq.~\eqref{eq: score} we have
    \[
    \mathrm{D}_\sigma^* (\mathrm{T}_g y) = \mathrm{T}_g y + \sigma^2 \mathrm{T}_g \nabla \log p_\sigma(y)
    \]
    \[
    \mathrm{D}_\sigma^* (\mathrm{T}_g y) = \mathrm{T}_g \left( y + \sigma^2 \nabla \log p_\sigma(y) \right) = \mathrm{T}_g \mathrm{D}_\sigma^* (y),
    \]
    this complete the proof.
\end{proof}

\subsection{Proof of Lemma~\ref{mylem: EBR}}
\textbf{Lemma (Equivariant Bias Reduction)} For $\forall x \in \mathcal{X}$ and $g \sim G$, if the bias $\mathcal{E}(x)$ generated by PnP is anisotropic, then the bias $ \tilde{\mathcal{E}}(x) $ generated by EPnP satisfies
    \begin{equation}
        \| \tilde{\mathcal{E}}(x) \|^2 < \int_G \| \mathcal{E}(\mathrm{T}_g x) \|^2 d\pi(g) = \mathbb{E}_{g} [\|\mathcal{E}(\mathrm{T}_g x)\|^2].
    \end{equation}
\begin{proof}
    First we have
    \begin{align*}
	\tilde{\mathcal{E}} (x) &= \int_G \mathrm{T}_g^T \cdot \hat{\mathrm{D}}_\sigma (\mathrm{T}_gx) d\pi(g) - \mathrm{D}_\sigma^*(x) \\
	&= \int_G \mathrm{T}_g^T \cdot \left[ \mathrm{D}_\sigma^*(\mathrm{T}_gx) + \mathcal{E}(\mathrm{T}_gx) \right] d\pi(g) - \mathrm{D}_\sigma^*(x) \\
    &= \int_G \left[ \mathrm{T}_g^T \cdot \mathrm{D}_\sigma^*(\mathrm{T}_gx)\right]d\pi(g) + \int_G \left[ \mathrm{T}_g^T \cdot \mathcal{E}(\mathrm{T}_gx) \right]d\pi(g) - \mathrm{D}_\sigma^*(x), 
    \end{align*}
    from Proposition~\ref{myprop: 2}, we have that the target denoiser $\mathrm{D}_\sigma^*$ is equivariant, i.e.,  $\mathrm{T}_g^{T} \cdot \mathrm{D}_\sigma^*(\mathrm{T}_g \cdot x) = \mathrm{D}_\sigma^*(x)$, then we can conclude that
    \[
    \tilde{\mathcal{E}}(x) = \mathbb{E}_{g \sim \pi}[] T_g^{-1} \mathcal{E}(T_gx)],
    \]
    using variance decomposition we have
    \[
    \|\tilde{\mathcal{E}}(x)\|^2 = \mathbb{E}_{g} [\|\mathrm{T}_g^{-1} \cdot \mathcal{E}(\mathrm{T}_g  x)\|^2] - \text{Var}_{g}(\mathrm{T}_g^{-1} \cdot \mathcal{E}(\mathrm{T}_g  x))
    \]
    since $\mathrm{T}_g$ is isometric transform, i.e. $\|\mathrm{T}_g^{-1} \cdot v\| = \|v\|$, thus above equation is equivalent to
    \[
    \|\tilde{\mathcal{E}}(x)\|^2 = \mathbb{E}_{g} [\| \mathcal{E}(\mathrm{T}_g  x)\|^2] - \text{Var}_{g}(\mathrm{T}_g^{-1} \cdot \mathcal{E}(\mathrm{T}_g x)),
    \]
    since bias $\mathcal{E}(x)$ is anisotropic, we have
    \[
    \text{Anisotropy} \implies \text{Var}_{g}(\mathrm{T}_g^{-1} \cdot \mathcal{E}(\mathrm{T}_g x)) > 0,
    \]
    thus
    \[
    \|\tilde{\mathcal{E}}(x)\|^2 = \mathbb{E}_{g} [\|\mathcal{E}(\mathrm{T}_g x)\|^2] - \underbrace{\text{Var}(...)}_{>0} < \mathbb{E}_{g} [\|\mathcal{E}(\mathrm{T}_g x)\|^2]
    \]
    this complete the proof.
\end{proof}

\subsection{Proof of Theorem~\ref{mythm: 3}}
the proof is similar to Theorem~\ref{them: 1}. At $k+1$-th iteration, for the output of mismatched denoiser $x_{k+1} = \tilde{\mathrm{D}}_\sigma (z_{k+1})$, by Lemma~\ref{lemma: 2}, we have
\begin{align*}
	\nabla F_{\lambda, \sigma} (x_{k+1}) &= \lambda \nabla f(x_{k+1}) + \nabla \tilde{\phi}_\sigma (x_{k+1}) \\
	&= \lambda \nabla f(x_{k+1}) + z_{k+1} - x_{k+1} - \tilde{\delta}_{k+1} \\
	&= \lambda \nabla f(x_{k+1}) + x_k - \lambda \nabla f(x_k) - x_{k+1} - \tilde{\delta}_{k+1} \\
	&= (x_k - x_{k+1}) + \lambda \left( \nabla f(x_{k+1}) - \nabla f(x_k) \right) - \tilde{\delta}_{k+1}, \\
\end{align*}
thus
\[
\nabla F_{\lambda, \sigma} (x_{k+1}) \leq \| x_k - x_{k+1} \| + \lambda \| \nabla f(x_{k+1}) - \nabla f(x_k) \| + \| \tilde{\delta}_{k+1} \|,
\]
since $\nabla f$ is $L_f$-Lipschitz, we have
\[
\| \nabla F_{\sigma, \lambda} (x_{k+1}) \| \leq (1+\lambda L_f) \| x_{k+1} - x_k \| + \| \tilde{\delta}_{k+1} \|,
\]
using $(\mathbf{a} + \mathbf{b})^2 \leq 2 \mathbf{a}^2 + 2 \mathbf{b}^2$ and Lemma~\ref{lemma: 2}, we have
\[
\| \nabla F_{\lambda, \sigma} (x_{k+1}) \|^2 \leq 2 (1+\lambda L_f)^2 \| x_{k+1} -x_k \|^2 + 4 \frac{\tilde{\epsilon}_{k+1}}{1-L},
\]
using Lemma~\ref{lemma: 1}, we have
\[
\| x_{k+1} - x_k \|^2 \leq \frac{2}{1 - L_f} \left( F_{\lambda, \sigma}(x_k) - F_{\lambda, \sigma}(x_{k+1}) \right) + \frac{2\tilde{\epsilon}_{k+1}}{1-L_f},
\]
combining these two relations,
\begin{align*}
\| \nabla F_{\lambda, \sigma} (x_{k+1}) \|^2 &\leq \frac{4(1+\lambda L_f)^2}{1-L_f} \left( F_{\lambda, \sigma} (x_{k}) - F_{\lambda, \sigma} (x_{k+1}) \right) + \frac{4(1+\lambda L_f)^2}{1-L_f} \tilde{\epsilon}_{k+1} + \frac{4 \tilde{\epsilon}_{k+1}}{1-L} \\
&\leq \frac{16}{1-L_f} \left( F_{\lambda, \sigma} (x_{k}) - F_{\lambda, \sigma} (x_{k+1}) \right) + \left( \frac{16}{1-L_f} + \frac{4}{1-L} \right) \tilde{\epsilon}_{k+1},
\end{align*}
summing over $k=0, 1, \ldots, t-1$ gives
\[
\sum_{k=1}^t \| \nabla F_{\lambda, \sigma} (x_k) \|^2 \leq \frac{16}{1-L_f} \left( F_{\lambda, \sigma} (x_{0}) - F_{\lambda, \sigma}^* \right) + \left( \frac{16}{1-L_f} + \frac{4}{1-L} \right) \sum_{k=1}^t \tilde{\epsilon}_{k}.
\]
Then by averaging over $t \geq 1$ iterations, we have
\[
\min_{1 \leq k \leq t} \| \nabla F_{\lambda, \sigma} (x_k) \|^2 \leq \frac{1}{t}\sum_{k=1}^t \| \nabla F_{\lambda, \sigma} (x_k) \|^2 \leq \frac{1}{t} \cdot \frac{16}{1-L_f} \left( F_{\lambda, \sigma} (x_{0}) - F_{\lambda, \sigma}^* \right) + \left( \frac{16}{1-L_f} + \frac{4}{1-L} \right)\frac{1}{t} \cdot \sum_{k=1}^t \tilde{\epsilon}_{k},
\]
then by Lemma~\ref{mylem: EBR}, we have
\[
\tilde{\epsilon}_k = \hat{\epsilon}_k - \text{Var}_{g}  < \hat{\epsilon}_k,
\]
which complete the proof.

\bibliography{main.bib}

\end{document}